\documentclass{article}

\usepackage{microtype}
\usepackage{graphicx}
\usepackage{booktabs} 

\usepackage[utf8]{inputenc} 
\usepackage[T1]{fontenc}    
\usepackage{hyperref}       
\usepackage{url}            
\usepackage{amssymb}
\usepackage{booktabs}       
\usepackage{amsfonts}       
\usepackage{nicefrac}       
\usepackage{microtype}      
\usepackage[utf8]{inputenc} 
\usepackage[T1]{fontenc}    
\usepackage{url}            
\usepackage{booktabs}       
\usepackage{amsfonts}       
\usepackage{microtype}
\usepackage{color}
\usepackage{paralist}
\usepackage{amsthm,amsmath}
\usepackage[ruled,vlined, linesnumbered]{algorithm2e}
\usepackage{graphicx}

\usepackage{times}
\usepackage{soul}
\usepackage{mathtools}

\usepackage{caption}
\usepackage{subcaption}
\usepackage{enumitem}
\usepackage{cleveref}

\usepackage{Definitions}

\usepackage{xcolor}

\newcommand{\argmin}{\operatornamewithlimits{argmin}}

\usepackage[accepted]{icml2021}

\icmltitlerunning{Overcoming Catastrophic Forgetting by Bayesian Generative Regularization}

\begin{document}

\twocolumn[
\icmltitle{Overcoming Catastrophic Forgetting by Bayesian Generative Regularization}



\icmlsetsymbol{equal}{*}

\begin{icmlauthorlist}
\icmlauthor{Patrick H. Chen}{ucla}
\icmlauthor{Wei Wei}{cloud}
\icmlauthor{Cho-jui Hsieh}{ucla}
\icmlauthor{Bo Dai}{brain}
\end{icmlauthorlist}

\icmlaffiliation{ucla}{Department of Computer Science, UCLA, California, USA}
\icmlaffiliation{cloud}{Google Cloud, Sunnyvale, California, USA}
\icmlaffiliation{brain}{Google Brain, Mountain View, California, USA}

\icmlcorrespondingauthor{Patrick H. Chen}{patrickchen@g.ucla.edu}

\icmlkeywords{Machine Learning, ICML}

\vskip 0.3in
]



\printAffiliationsAndNotice{}  

\begin{abstract}
The streaming update of Bayesian posterior calculation provides us a natural way for continual learning. However, the na\"{i}ve mean-field posterior parametrization for variational approximation is inappropiate in neural network, and thus, lock the full ability for preventing catastrophic forgetting. 
To resolve this issue, we introduce a generative regularization for all given classification models, which is implemented by leveraging energy-based models with contrastive loss, to obtain the sufficient features for valid decomposition in posterior approxiamtion.
By combining discriminative and generative loss together, we empirically show that the proposed method outperforms state-of-the-art methods on a variety of tasks, avoiding catastrophic forgetting in continual learning. In particular, the proposed method outperforms baseline methods over 15$\%$ on the Fashion-MNIST dataset and 10$\%$ on the CUB dataset.
\end{abstract}

\vspace{-3mm}
\section{Introduction}
\label{section:introduction}
\vspace{-2mm}

Many real-world machine learning applications require classification models to learn a sequence of tasks in an incremental way. For each task, learning system could only access part of whole data and the previously seen data can not be assessed. For example, previous customer data usually can not be accessed due to increasingly more strict data regulations on the user privacy, such as GDPR \cite{voigt2017eu}. The labelled data of an existing task can be depleted when new tasks emerge~\cite{sutton2014online,kirkpatrick2017overcoming}.
Thus, an intelligent agent for continual learning must not only adapt to newly incoming tasks but also perform well on entire set of all the existing tasks in an incremental way that avoids revisiting all previous data at each stage.
Previous studies~\cite{mccloskey1989catastrophic,ratcliff1990connectionist} found that conventional deep learning models fail to tackle continual learning with the phenomenon of \textbf{catastrophic forgetting}, where deep neural networks tend to lose the information of previous tasks (i.e. classification accuracy drops significantly) after a new task is introduced.

Apparently, in order to achieve continual learning, catastrophic forgetting is an important issue to be addressed. A common strategy is to fix parameters used in the previous tasks. When new tasks arrive, based on different criteria, each method can reuse part of the fixed parameters \cite{DBLP:journals/corr/FernandoBBZHRPW17}, expand some model components \cite{DBLP:journals/corr/RusuRDSKKPH16,yoon2018lifelong}, or search for the best new model architecture to process new task \cite{DBLP:journals/corr/abs-1904-00310}. Alternatively, instead of fixing a model, memory-based methods store a subset of previous data and constrain the update of models by leveraging the distilled knowledge from previous tasks~\cite{castro2018end,hou2018lifelong,javed2018revisiting,li2017learning,rebuffi2017icarl,shin2017continual}. These methods demonstrate the capability of alleviating the forgetting in practice on several datasets, however, without investigating and explaining the potential cause of catastrophic forgetting. More importantly, 
model adaption methods come at the cost that the model size expands correspondingly to the number of new tasks; while keeping data directly violates GDPR regulation. These drawbacks make existing methods not applicable for large-scale real-world applications.
Therefore, there is a need to investigate the cause of catastrophic forgetting for a principled algorithm under the memoryless, fixed model setup. 

Most of the existing literature 
views the incremental training as a moving path in parameter space, then the catastrophic forgetting happens when the update direction obtained based on partial data leads an inappropriate solution. Therefore, it is natural to design the search and update directions in training to avoid the catastrophic forgetting~\citep{kirkpatrick2017overcoming,v.2018variational,Zenke:2017:CLT:3305890.3306093,smola2003laplace}. 
Variational Continual Learning (VCL)~\citep{v.2018variational}, as a representative algorithm, exploits the equivalent streaming update form of Bayesian posterior calculation, which by nature only uses part of data, and therefore can combat forgetting. In practice, the exact posterior is intractable, especially in the Bayesian neural network, then, variational methods are used for approximation.
The VCL achieves good empirical performance on various benchmarks. However, VCL approxiamtes the posterior distribution by assuming parameters shared by all tasks to be \emph{independent} of all task-specific parameters, which is difficult to satisfy, especially in neural network, as we will illustrate in~Figure~\ref{fig:problem}. Moreover, our experimental results demonstrate that discrminative VCL models tend to extract features from limited parts of an object, which is only useful particular current task, instead of diverse features from all different parts. Since the classifier is built on concentrated features, independence assumption in VCL is prone to errors as training in the subsequent tasks might make the model attend to other features which are not considered in the earlier tasks.
These drawbacks of VCL motivate us to have a valid posterior approximation while encouraging models to focus on more diverse features. 

Fortunately, we can fullfil these two desiderata by equipping the model with data generative regularization in the training process. 
The generative regularization is pushing the model to catch the characteristics of all parts of the object for generatation, so that the shared component will be sufficient features and stable across all tasks. Meanwhile, with the sufficient features, we can recover the independent condition as we discusse in Section~\ref{section:baye}. 
Our contributions can be summarised as follows: 
\begin{itemize}[leftmargin=*, noitemsep,topsep=0pt,parsep=0pt,partopsep=0pt]
	\item  we analyze Bayesian approach in the continual learning setup and point out a deficiency of the parameter independence assumption; 
	\item  we propose to use energy-based model with Langevin dynamic sampling as an implicit regularization term in training discriminative task;
	\item  we empirically show that the proposed variational learning with generative regularization works well on all benchmark datasets we consider.
\end{itemize}

\begin{figure}[t]
\vspace{-2mm}
\centering
    \includegraphics[width=0.8\linewidth]{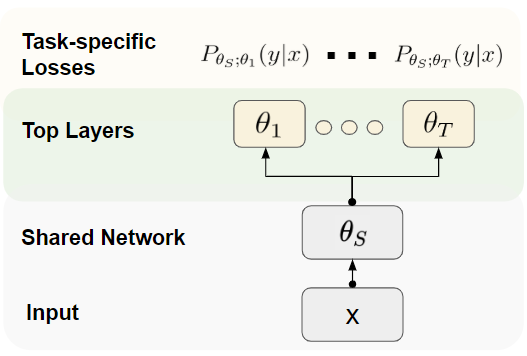}
    \vspace{-3mm}
    \caption{Illustration of the problem and model setup. }
    \label{fig:problem}
    \vspace{-2mm}
\end{figure}

\section{Related Work}\label{section:relatedwork}

\noindent \textbf{Continual learning by regularization.} There are a rich body of methods solving catastrophic forgetting problem by introducing different regularizations. EWC~\citep{kirkpatrick2017overcoming} aims to minimize the change of weights that are important to previous tasks through the estimation of diagonal empirical fisher information matrix. SI \cite{Zenke:2017:CLT:3305890.3306093} proposes
to alleviate catastrophic forgetting by allowing individual
synapse to estimate their importance for solving learned
tasks, then penalizing changes on the important weights. IMM \cite{DBLP:journals/corr/LeeKHZ17} trains individual models on each task and then carries out a second stage of training to
combine them. VCL \cite{v.2018variational} takes a Bayesian point of view to model a sequential learning procedure. This line of research assumes a memoryless (i.e., no stored old-data) and fixed model (i.e., model architecture cannot be adjusted during training) setup to study inherent causes of catastrophic forgetting. Our work falls in this line of research, but is derived in a principled way, and we mainly compare our algorithm with state-of-the-art methods in this category.

\noindent \textbf{Continual learning by model adaption.} Another class of methods addresses the continual learning problem by allowing the model to expand its capacity, while keeping the parameters used to solve previous tasks fixed. 
PathNet \cite{DBLP:journals/corr/FernandoBBZHRPW17} selects paths between predefined modules, and tuning is allowed only when an unused module is selected. Dynamically expandable networks (DEN) \cite{yoon2018lifelong} selects whether to expand or duplicate layers based on certain criteria for an incoming new task. Similarly, \citet{DBLP:journals/corr/RusuRDSKKPH16} tries to leverage the strategies adopted in progressive networks to heal forgetting. Following this line of research, \cite{DBLP:journals/corr/abs-1904-00310} proposed to solve the continual learning by explicitly taking into account
continual structure optimization via differentiable neural
architecture search. Our main goal is to study the catastrophic forgetting problem given the constraint that the structure of underlying model is fixed, while this category is out of our consideration. 

\noindent \textbf{Memory-based approaches and generative models.} Previous works also try to alleviate catastrophic forgetting by introducing memory systems which store previous data and replay the stored old examples with the new data \cite{farquhar2019unifying,li2018supportnet,lopez2017gradient,rebuffi2017icarl,robins1995catastrophic}. Specifically, these approaches require to keep either a coreset data or a generative model to replay previous tasks in order to leverage the distilled knowledge from previous tasks \cite{castro2018end,hou2018lifelong,javed2018revisiting,li2017learning,rebuffi2017icarl,shin2017continual,wu2019large}. In practice, these methods alleviate the forgetting phenomena if enough old-data recorded, but it will increase the data usage
Since our method is related to generative models, we will also compare to one of the representative algorithms, variational generative replay (VGR) \cite{farquhar2019unifying}.


\noindent {\bf Energy-based model.} Our method is partly based on applying energy-based models~(EBMs).  We refer readers to \cite{lecun2006tutorial} for a more comprehensive review. The primary difficulty in training EBMs comes from estimation of the partition function. Our work follows the derivation in \cite{dai2019exponential}. We notice that some concurrent works have also pointed out the importance of generative capability in the training process \cite{du2019implicit,grathwohl2019your}, the motivation behind these works differ from us and their focus is not in overcoming catastrophic forgetting. Furthermore, empirical results showed that using only EBMs could not achieve the best performance. The proposed integration of Bayesian framework and generative capability significantly outperforms EBM alone.

\vspace{-2mm}
\section{Methods}

In this section, we first clarify the problem setting in~Section~\ref{section:problem}. Then, we analyze the drawback of the  posterior approximation used in the original VCL~\citep{v.2018variational}, which motivates the generalization regularizer in Section~\ref{section:baye}. After providing a brief introduction to EBM in~Section~\ref{sec:ebm}, we design the generation-regularized Bayesian EBM to combat catastrophic forgetting in Section~\ref{sec:gen_reg}. 

\vspace{-2mm}
\subsection{Problem Statement}
\label{section:problem}

A given classification model $M$, with a set of parameters denoted as $\theta$, consists of parameters shared across all tasks $\theta_S$ and parameters dedicated to specific tasks $\theta_t$. Sequential tasks are denoted as $D_1,D_2,\dots,D_T$, where each $D_t=(X_t,Y_t)$  defines a classification task with observations $X_t$ and labels $Y_t$. In the canonical setup \cite{kirkpatrick2017overcoming,v.2018variational}, for each task $t$, only one dataset $D_t$ can be used and all previous datasets $D_1,\dots,D_{t-1}$ cannot be accessed. The goal of our work is to achieve good classification accuracy on each task after observing all $T$ tasks. In addition, we do not allow the algorithm to change the pre-defined structure of the model $M$ or introduce additional parameters in shared networks. An illustration of the problem formulation is shown in Figure \ref{fig:problem}.


\vspace{-2mm}
\subsection{Motivations}
\label{section:baye}
We first explain why Bayesian method is a good candidate to resolve the forgetting problem, and then point out what lacks in existing literature, which motives the EBM view with generalization regularization. 

Following \cite{v.2018variational}, we assume some prior distribution of model parameters $p_{0}(\theta)$ (e.g., $p_{0}(\theta)$ follows normal distribution). According to Bayes' rule, the posterior distribution after observing $T$ datasets can be written as: 
\resizebox{1\linewidth}{!}{
  \begin{minipage}{\linewidth}
\vspace{-2mm}
\begin{align*}
    p(\theta|D_{1:T}) &\propto p(\theta) \prod_{t=1}^{T} p(D_t|\theta) \propto \bigg( p(\theta) \prod_{t=1}^{T-1} p(D_t|\theta) \bigg) p(D_T|\theta) \\
    &= \bigg( p(\theta) p(D_{1:T-1}|\theta) \bigg) p(D_T|\theta) \\
    &\propto p(\theta | D_{1:T-1}) p(D_T|\theta).
\end{align*}
  \end{minipage}}
Therefore, we can see that if we have a good posterior approximation of previous tasks $p(\theta | D_{1:T-1})$, by Bayesian approach we can combine $p(\theta | D_{1:T-1})$  and likelihood of the current task $p(D_T|\theta)$ to obtain the posterior of model parameters $p(\theta|D_{1:T})$ that work well for all tasks. The above decomposition paves a natural way for Bayesian method to handle the continual learning setup.  
In general the posterior is intractable, however, we can approximate the true posterior $p(\theta|D_{1:t})$ of each task $t$ by KL-divergence via variational inference, such that $\forall t = 1,2,\ldots,T$, 
\begin{align*}
q_{t}(\theta) = \argmin_{q \in \bQ} KL \big( q(\theta) \| \frac{1}{Z_t}q_{t-1}(\theta)  p(D_t|\theta)\big), 
\end{align*}
where $q_{t}(\theta)$ and $q_{t-1}(\theta)$ are the approximated posterior up to timestamp $t-1$ and $t$, $\bQ$ is a predefined approximate posteriors set and $Z_t$ is a normalization constant which needs not to be computed. We then apply variational method to estimate the lower bound of $P(y | \theta,x)$ and arrive the following training loss for each task $t$:
\vspace{-2mm}
\begin{align}\label{eq:VCLloss}
	\resizebox{0.9\columnwidth}{!}
	{
	$
	\sum_{n=1}^{B_t} \EE_{\theta \sim q_{t}(\theta)} [ -\log p(y_{t,n}|\theta,x_{t,n})] + KL\big(q_{t}(\theta)\|q_{t-1}(\theta)\big),
	$
	}
\end{align}
where $B_t$ denotes the dataset size of task $t$. One can parametrize $p(y_{t,n}|\theta,x_{t,n})$ with Gaussian distribution and softmax upon the output of neural network, for regression and classification, respectively. The parameter of posterior can be trained end-to-end via reparametrization trick~\citep{kingma2014stochastic}. 

Despite that Bayesian method looks promising, we need to point out one important deficit 
in VCL~\cite{v.2018variational}. VCL assumes the shared model parameters $\theta_{S}$ are \emph{independent} of the individual head network $\theta_{t}$ and thus the posterior function $ p(\theta|D_{1:t})$ of the task $t$ could be decomposed into:
\begin{equation}
\label{eq:vcl}
p(\theta|D_{1:t}) =  p(\theta_{t}|D_{1:t}) p(\theta_{S}|D_{1:t}),
\end{equation}
where $\theta = \{\theta_S,\theta_t \}$. VCL then applies Bayeisan approach on approximating $p(\theta_{S}|D_{1:t})$ and fix $\theta_{t}$ after training each task t. However, the independence assumption between $\theta_t$ and $\theta_S$ is not true in general, especially in neural network where the head parameter $\theta_t$ highly depends on the shared layers. The correct factorization of posterior function should be
\begin{equation}
\label{eq:correct}
p(\theta|D_{1:t}) =  p(\theta_{t}|D_{1:t};\theta_{S}) p(\theta_{S}|D_{1:t}),
\end{equation}
where the dependence between $\theta_S$ and $\theta_t$ exists. Thus, in order to correctly apply Bayesian framework, we may use a sufficient feature for $\theta_S$ such that the equation \eqref{eq:vcl} becomes a valid reduction.

We provide a sufficient condition under which the decomposition~\eqref{eq:vcl} is valid. 
\begin{proposition}\label{thm:equivalence}
If the causal joint distribution is
\[
p(\theta_t, \theta_S, D_{1:t}) = p(\theta_S)p(D_{1:t}|\theta_S)p(\theta_t|D_{1:t}),
\] 
we have the $\theta_S\perp \theta_t|D_{1:t}$ in posterior. 
Therefore, the decomposition~\eqref{eq:vcl} becomes valid. 
\end{proposition}
\begin{proof}
The conclusion can be verified straightforwardly. 
\begin{eqnarray*}
q(\theta_S, \theta_t|D_{1:t}) &=& \frac{p(\theta_S)p(D_{1:t}|\theta_S)p(\theta_t|D_{1:t})}{p(D_{1:t})}\\
&=& p(\theta_S|D_{1:t})p(\theta_t|D_{1:t}),
\end{eqnarray*}
from which we obtain the conditional independence property $\theta_S\perp \theta_t|D_{1:t}$. 
\end{proof}
The condition in~Proposition~\ref{thm:equivalence} inspires us the generative requirement on the intermediate layer of the model (i.e., $\theta_{S}$) to revalid the decomposition~\eqref{eq:vcl}. Bringing generative power ((i.e., $p(D_{1:t}\theta_S)$) into the play could exploit variational Bayesian inference overcome forgetting better. 

The remaining question is 
\begin{center}
 \emph{how do we equip the underlying model with generative power without adding more number of parameters? }
\end{center}
We answer this question by resorting to the Energy-based model~(EBM). Essentially, the neural network $p(y|x,\theta)$ can be understood as a EBM, which automatically has the discriminative and generative ability, although in most of the training for supervised tasks, the generative ability is just simply ignored. 
From this perspective, we will complement the existing discriminative loss in the training with an additional generative loss term to ensure the causal condition in~\eqref{thm:equivalence}, and eventually guarantee the decomposition in Bayesian inference is valid. Therefore, it totally release the power of Bayesian inference to overcome catastrophic forgetting. In the following sections, we illustrate how generative power of EBM can be fit into the Bayesian method.




\begin{figure}[t!]
\centering
    \includegraphics[width=1\linewidth]{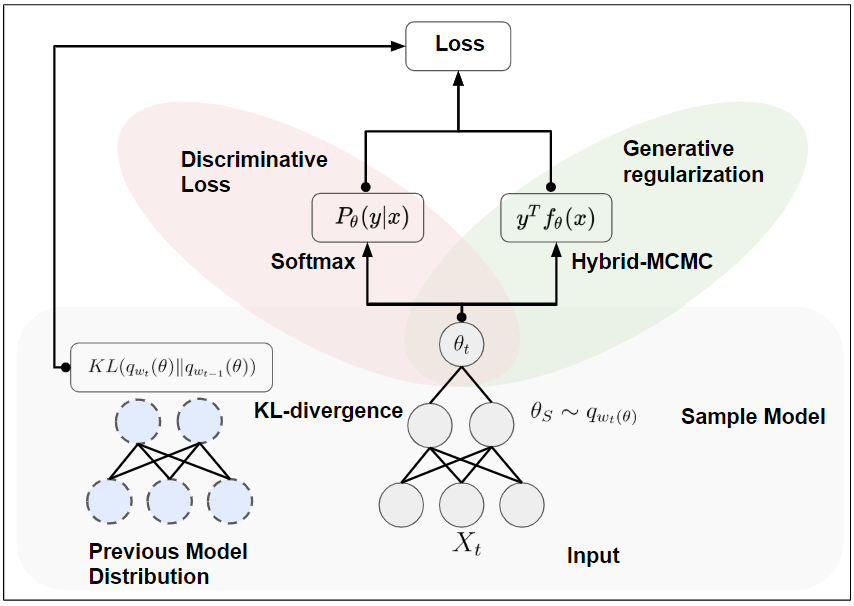}
    \caption{Illustration of the proposed method. }
    \label{fig:ourmethod}
\end{figure}

\subsection{Energy-based Model}\label{sec:ebm}


For any given discriminative model $f_{\theta}(x)$  (e.g., deep neural networks for classification tasks) parameterized by $\theta$ as 
\[
p(y|x) = \frac{\exp(y^\top f_\theta(x))}{Z_x}
\] 
with $Z_x(\theta) = \sum_{y\in \mathcal{Y}} \exp(y^\top f_\theta(x))$, it can be view as EBM with energy function $y^\top f_\theta(x)$. Then, obviously, by redefine the partition function, we obtain the joint distribution with generative ability:
\vspace{-3mm}
\begin{equation}
    p_{\theta}(x,y) = \frac{\exp( y^\top f_{\theta}(x))}{Z(\theta)},
\end{equation}
where $Z(\theta) = \sum_y{ \int \exp(y^\top f_{\theta}(x))dx }$. In this work, $f_{\theta}(\cdot)$ is a neural network parameterized by $\theta$. We can train the joint EBM by maximum likelihood estimation:
\vspace{-2mm}
\begin{align}
    \max_{\theta} p_{\theta}(x,y) &= \max_{\theta} \log p_{\theta}(x,y)\nonumber \\
    &= \max_{\theta} y^\top f_{\theta}(x) - \log Z(\theta). 
\end{align}
However, directly solving MLE of general EBMs is intractable due to the $\log$-partition function $\log Z(\theta)$. To alleviate the computation, Contrastive Divergence (CD) is proposed in \cite{hinton2002training}. CD estimates the gradient of the MLE of EBM as:
\begin{equation}\label{eq:CD}
	\resizebox{0.95\columnwidth}{!}
	{
	$
	  \nabla_\theta \log p_\theta(x,y) = \EE_D \big[ y^\top \nabla_\theta f_{\theta}(x) \big] - \EE_{p_{\theta}(x,y)} \big[ y^\top \nabla_\theta f_{\theta}(x) \big],
    $
    }
\end{equation}
where $p_{\theta}(x,y)$ denotes the underlying distribution from EBM. The second term $\EE_{p_{\theta}(x,y)} \big[ y^\top \nabla_\theta f_{\theta}(x) \big] $ can be calculated as firstly sample (batch of) data $x_t,y_t$ by using Langevin dynamic sampling shown in Algorithm \ref{alg:langevin}, and then calculate the $y_{t}^\top f(x_t)$ to stochastically get the estimated value of $\EE_{p_{\theta}(x,y)} \big[ y^\top \nabla_\theta f_{\theta}(x) \big] $. 


\begin{algorithm}[!t]
   \caption{Gibbs-Langevin Dynamic Sampling}
   \label{alg:langevin}
\begin{algorithmic}
   \STATE {\bfseries Input:} Buffer $B$, storing previous sampled data
   \STATE {\bfseries Output:} Sampled data $x_S,y_S$ and buffer B

  \STATE  $x_0 \sim B$ \\
  \FOR{$s=1$ {\bfseries to} $S$}
  \STATE $y_s \sim p(y | x_s)$  
  \STATE $x_s = x_{s-1} + \frac{1}{2}\eta_{s} \nabla_{x} [ y_s^\top f(x_{s-1})] $ + $\epsilon$, \\ $\epsilon \sim N(0,\eta_{s})$ 
  \STATE $\eta_s = \frac{1}{s}$
   \ENDFOR
  \STATE Add $x_S,y_S$ into $B$. 
  \STATE Return $x_S,y_S,B$.
\end{algorithmic}
\end{algorithm}

\begin{algorithm}[t!]
   \caption{Algorithm of Bayesian Generative Regularization (BGR) at task $t$.}
   \label{alg:overall}
	\begin{algorithmic}
   \STATE {\bfseries Input:} Dataset of task $t$ $D_{t}$, Posterior distribution of previous tasks  $q_{t-1}(\theta)$, Number of training epochs $E$ and learning rate $\beta$
   \STATE {\bfseries Output:} Posterior distribution $q_{t}(\theta)$ of learned model \\
     $q_{t}(\theta) = q_{t-1}(\theta)$

    \FOR{$epoch=1$  {\bfseries to}  $E$}
    \STATE $x_{b},y_{b} \sim D_{t}$ 
    \STATE $\theta \sim q_{t}(\theta)$ 
    \STATE Generate sample $x_{t}, y_{t}$ by Algorithm \ref{alg:langevin} 
    \STATE Calculate gradient $\nabla_\theta L(\theta; p_\theta)$ via Theorem \ref{thm:gradient_est}. 
    \STATE $q_{t}(\theta)$ = $q_{t}(\theta)$ - $\beta \nabla_\theta L(\theta; p_\theta)$  
     \ENDFOR
    \STATE  Return $q_{t}(\theta)$
\end{algorithmic}
\end{algorithm}

\subsection{Bayesian Inference as Learning with Generative Regularization}\label{sec:gen_reg}
With the formulation of generative loss, instead of interpreting $p(D_t|\theta)$ as a discriminative model $p(y_t|x_t;\theta)$ in eq \eqref{eq:VCLloss}, we have $p(D_t|\theta)$ to be a generative model as $p_\theta(x,y)$. Now, instead of lower-bounding $P(y | \theta,x)$, we estimate the lower bound of $P(x,y,\theta)$ and the core training objective of task $t$ in variational method changes from  eq \eqref{eq:VCLloss} into :
\resizebox{1\linewidth}{!}{
  \begin{minipage}{\linewidth}
\begin{align}
  \min_{q_{t} \in \bQ} \EE_{q_t,D_t} \big[ - \log p_\theta(x,y) + KL \big(q_{t}(\theta | D_{1:t})  \| q_{t-1}(\theta | D_{1:t-1}) \big)\big]&, \nonumber
\end{align}
  \end{minipage}
}

where $\bQ$ is the functional space of posterior distribution. For simplicity, we follow the literature to assume $\bQ$ to represent mean-field distribution, and we  generate a model parameter $\theta$ by sampling it from $q_{t}$. 
Recall that $p(x, y) = p(y|x)p(x)$, thus we can rewrite the objective as 
\begin{multline}\label{eq:totalObj}
        \min_{q_{t} \in \bQ} \EE_{q_t,D_t} \big[ -(1 - \lambda) \log p_\theta(y, x) -\lambda \log p_\theta(y | x) \\
        - \lambda\log p_{\theta}(x)+ KL(q_{t}(\theta | D_{1:t})  \| q_{t-1}(\theta | D_{1:t-1})) \big],
\end{multline}
where the $\log p_{\theta}(y|x)$ can be understood as the common discriminative loss, while both $\log p_{\theta}(x)$ and $\log p_{\theta}(x,y)$ can be understood as generative regularizations that match the empirical joint distribution and marginal distribution simultaneously. 
Then, we apply Contrastive Divergence to optimize~\eqref{eq:totalObj}, which providing the estimation of gradient of $p_{\theta}(x,y)$. In fact, new objective is also related to contrastive loss~\cite{chen2020simple,dai2019exponential}, where we use the synthesis samples as negative samples. Here we give a derivation of unbiased gradient estimator of $\log p_{\theta}(x)$ in the following theorem. 
\begin{theorem}
\label{thm:gradoflogp}
Given a discriminative model $f_{\theta}(x)$, the unbiased gradient estimator of the  corresponding Energy-based model term $\log p_{\theta}(x)$  is given by the following estimator:
\resizebox{1\linewidth}{!}{
  \begin{minipage}{\linewidth}
  
\begin{align*}
\nabla_{\theta} \log p_{\theta}(x) = 
      \EE_{p_{\theta}(y|x)}[y^\top \nabla_\theta f_{\theta}(x)] -  \EE_{p_{\theta}(x,y)}[y^\top \nabla_\theta f_{\theta}(x)].
\end{align*}

  \end{minipage}}
  
\end{theorem}
\begin{proof}
The proof is postponed to Appendix~\ref{appendix:A}.
\end{proof}
Based on this theorem, we could obtain the derivative of the objective in eq \eqref{eq:totalObj} by using eq \eqref{eq:CD} and Theorem \ref{thm:gradoflogp}, and we summarize it in the following theorem. 
\begin{theorem}
\label{thm:gradient_est}
The estimation of gradient of loss used in training the proposed method eq \eqref{eq:totalObj} is given by
\begin{align*}
      \nabla_\theta L(\theta; p_\theta) \nonumber &\triangleq  
     \frac{1}{\lambda}\nabla_{\theta}KL(q_{t}(\theta | D_{1:t})  \| q_{t-1}(\theta | D_{1:t-1})) \\ 
&  \nonumber - \nabla_{\theta} \log p_\theta(y | x)  \\  &  \nonumber +\frac{1}{\lambda}(\EE_{p_{\theta}(x,y)} \big[ y^\top \nabla_\theta f_{\theta}(x) \big] - \EE_{D_t}[{y}^\top  \nabla_\theta f_{\theta}(x)]). 
\end{align*}
\end{theorem}
\begin{proof}
The proof is postponed to Appendix~\ref{appendix:B}.
\end{proof}

The overall illustration of losses used in this work is summarized in Figure~\ref{fig:ourmethod}. The first term of the gradient estimation corresponds to the weighted KL-divergence between posterior approximation of task $t$ and $t-1$. The second term is the common NLL loss used in training deep neural networks. The calculation of these two terms corresponds to the gradient of forward neural network computation, and thus it could be obtained by back-propagation of underlying model $f_\theta$. The rest two terms correspond to the weighted generative capability. In this paper, we treat the generative term as a regularization term. $\lambda$ represents the importance balance of the generative regularization and discriminative NLL loss. We named the proposed method \textbf{Bayesian Generative Regularization}, and the  overall algorithm is summarized in Algorithm~\ref{alg:overall}.

\section{Experiments}
\subsection{Datasets}
We evaluate the proposed method on four datasets.

\noindent \textbf{Permuted-MNIST} Permuted-MNIST is a very popular benchmark dataset in the continual learning literature. The dataset received at each time step $D_t$ consists of labeled MNIST images whose pixels have undergone a fixed random permutation.

\noindent \textbf{Split-MNIST} This experiment was used by \cite{Zenke:2017:CLT:3305890.3306093}. Five binary classification tasks from the MNIST dataset arrive in sequence: 0/1, 2/3, 4/5, 6/7, and 8/9.


\noindent \textbf{Fashion-MNIST} Fashion-MNIST \cite{xiao2017/online}, similar to MNIST dataset, 
consists of a training set of 60,000 examples and a test set of 10,000 examples. Each example is a 28 x 28 grayscale image, associated with a label from 10 classes. This dataset represents more realistic features of real-world images and thus becomes an increasingly popular benchmark. For this task, we follow the Split-MNIST setup to split the classes into sequence: 0/1 (T-shirt/Trouser), 2/3 (Pullover/Dress), 4/5 (Coat/Sandal), 6/7 (Shirt/Sneaker), and 8/9 (Bag/Ankle boot).

\noindent \textbf{CUB} To further validate the proposed method could work on real-world color images, we perform experiments on Caltech-UCSD Birds (CUB) dataset. CUB is an image dataset with photos of 200 bird species. We select top 100 classes with more training images and then split thsese 100 classes into 10 continual learning tasks randomly. Each task consists of 5 binary classification in order. Detailed processing of the dataset is described in the supplementary.

\subsection{Baseline Methods}

We compare our method to the  following baseline methods.

\begin{itemize}[leftmargin=*]
  \item SGD: simply trains each task in an incremental setup without any regularization. It serves as the bottom line of all the methods. 
  \item All-data: trains the tasks jointly assuming all datasets are available. At each step, a random dataset is sampled and then a batch of data is sampled from the dataset. It serves as the upper bound and indicates the difficulty of the classification task.
  \item EWC \cite{kirkpatrick2017overcoming}: builds the importance estimation on top of diagonal Laplace propagation by calculating the diagonal of empirical Fisher information.
   \item VCL \cite{v.2018variational,swaroop2019improving}: conducts variational inference from Bayesian point of view of continual learning. VCL is reported as the most competitive method under our problem setup. In particular, we implement the improved version of VCL \cite{swaroop2019improving}.
      \item VGR \cite{farquhar2019unifying}: extends VCL by augmenting a GAN generative model to record the replay data.

\end{itemize}

Detailed processing of the dataset, implementation of the baseline methods and hyperparameters of the proposed method are described in the supplementary. 

\subsection{Results and Analysis}

The evaluation metric used is average classification accuracy over all observed tasks. We first summarize accuracy of each method after observing all tasks in Table~\ref{tab:Prelim}. Our proposed method is named \textbf{Bayesian Generative Regularization (BGR)}. We notice that ``All-data'' achieves high accuracy for almost all datasets. Accuracy on CUB drops a bit as there are certain species of birds which are difficult to classify it correctly. This shows that all classification tasks are not difficult when all data are provided. The challenges are indeed faced when continual learning setup comes in and causes forgetting.  
In Table~\ref{tab:Prelim}, we see that BGR outperforms baselines in all tasks. In particular, the improvement is significant on Fashion-MNIST and CUB dataset which contain more real-world alike objects. BGR increases about 15$\%$ accuracy in Fashion-MNIST and 10$\%$ in CUB datasets.

  \begin{table}[t]
    \begin{minipage}{.5\textwidth}
      \centering
    \begin{tabular}{|c|c|c|c|c|} \hline
          &  Permuted & Split & Fashion & CUB  \\\hline
         All-data & 99.3& 99.1&99.3&88.3 \\\hline
         SGD & 37 & 90 & 74.6 & 65.2  \\\hline
         EWC & 87.5 & 97.4 & 82.2 & 67.2  \\\hline
         VCL & 92.3 & \bf{98.2} & 76.9 & 67.4 \\ \hline
         VGR & 70.5 & 97.7 & 86.2 & 76.8 \\ \hline
         BGR & \bf{92.7} & \bf{98.2} & \bf{97.2} & \bf{78.8}\\ \hline
    \end{tabular}
        \caption{Summarization of overall performance on continual learning tasks. Results shown in the table are average classification accuracy (in $\%$) of each task.} 
    \label{tab:Prelim}
\end{minipage}
  \end{table}

\begin{figure}[t!]
\centering
    \includegraphics[width=.95\linewidth]{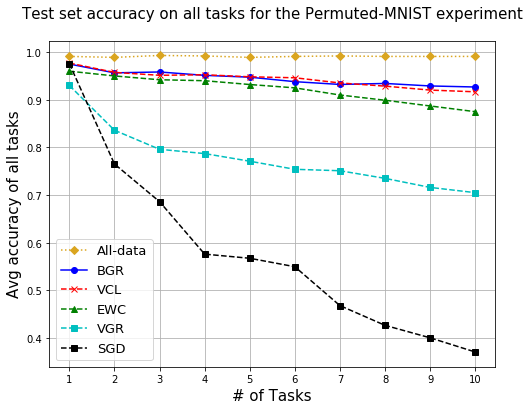}
    \caption{Detailed Classification Results of Permuted-MNIST. }
    \label{fig:permutedmnist}
\end{figure}

\begin{figure}[t!]
\centering
    \includegraphics[width=.9\linewidth]{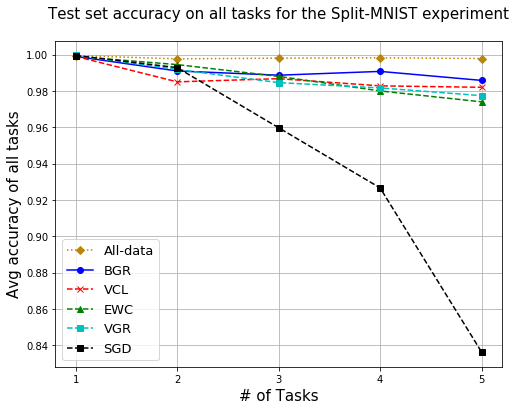}
    \caption{Detailed Classification Results of Split-MNIST. }
    \label{fig:splitmnist}
\end{figure}

In addition to accuracy after observing all tasks, we are also interested in individual performance after observing each new incoming task. Average classification accuracy of each time step of Permuted-MNIST and Split-MNIST are shown in Figure~\ref{fig:permutedmnist} and Figure~\ref{fig:splitmnist}. We can observe that despite the performance of SGD-only training drops abruptly, all other methods performs relatively steady over all time steps, and BGR stands out in the later time steps. For real-world objects as Fashion-MNIST and CUB, results are shown in Figure~\ref{fig:fashion} and Figure~\ref{fig:bird}. These two tasks contain more difficult classification tasks and thus it is more challenging when posed as continual learning setup. The difficulty of each task might be very different hence the accuracy fluctuates. Consequently, the curve will not be as smooth as previous two datasets. Nevertheless, we again observe that BGR has a relatively steady performance over the baseline methods.
\begin{figure}[t!]
\centering
\includegraphics[width=.95\linewidth]{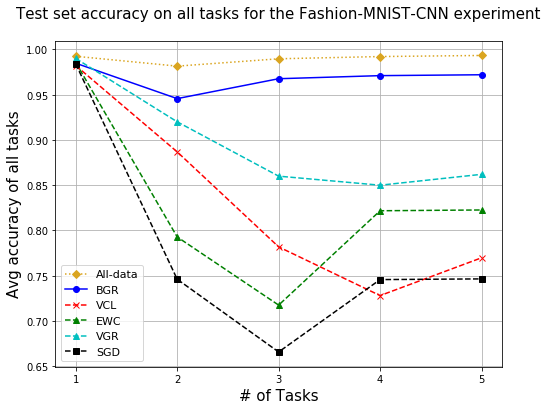}
    \caption{Detailed Classification Results of Fashion-MNIST. }
    \label{fig:fashion}
\end{figure}

\begin{figure}[t!]
\centering
    \includegraphics[width=.95\linewidth]{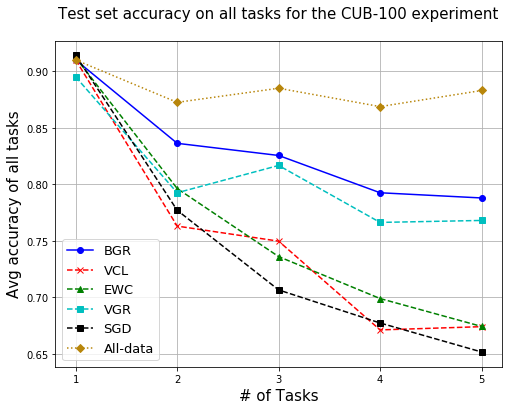}
    \caption{Detailed Classification Results of CUB. }
    \label{fig:bird}
\end{figure}
We also want to point out that since BGR contains generative capability, indeed we could sample images from the trained model. However, our main focus is not generative model but overcoming forgetting, so the generated images might not be realistic as the state-of-the-art generative models. The generative power used rather as a regularization to make the model robust to continual learning setup. We have attached some examples of generated images of MNIST and Fashion-MNISt dataset in the supplementary.

\subsection{Ablation Study}
Despite we have validated the performance of the proposed method, we are not sure if the gain comes from generative regularization, Bayesian method or indeed the better estimation of the posterior by combining two approaches. Therefore, we conduct ablation analysis on Fashion-MNIST and CUB to verify the importance of each component. Notice that BGR without the generative regularization would simply become the VCL method. To test the generative component without Bayesian framework, we will remove model parameter sampling procedure and KL-divergence term. This leads to normal training of the classifier with NLL loss and the generative loss from EBM. We denote this setup as GEN. We also try to apply the GEN with L2 regularization which resembles the KL divergence term in our formulation. We denote such method as GEN-L2. Results of all the methods are summarized in Table~\ref{tab:abalation}.

\begin{table}[t]
          \centering
    \begin{tabular}{|c|c|c|} \hline
          &  Fashion & CUB  \\\hline
         SGD & 74.6 & 65.2  \\\hline
         All-data & 99.3&88.3 \\\hline
         GEN &  87.9 & 74.0  \\\hline
         GEN-L2 & 90.9 & 72.8  \\\hline
         VCL & 76.9 & 67.4 \\ \hline
         BGR &  \bf{97.2} & \bf{78.8}\\ \hline
    \end{tabular}
    \caption{Ablation study of overall performance on Fashion-MNIST and CUB datasets. Results shown in the table are average classification accuracy (in $\%$) of each task.}
    \label{tab:abalation}
 \end{table}

From Table~\ref{tab:abalation}, we could observe that generative term itself is very important to overcome the catastrophic forgetting. Compared to the performance of VCL, GEN could achieve more than 5$\%$ performance gain on CUB, and more than 10$\%$ on Fashion-MNIST. Generative capability indeed provides a more robust model in continual learning setup and this validates our initial intuition that knowing the complete formulation of the object would make model perform better. On the other hand, results show that adding L2 regularization on top of generative term is not necessarily helpful. Even when it's effective, the performance gain is rather limited. 

However, we also notice that the generative term alone, without Bayesian updates, cannot reach a performance comparable to the proposed method. Furthermore, we could observe the synergy effect that the sum of the performance gain from VCL and GEN together could not reach the performance of BGR. This implies that in BGR, Bayesian framework and generative term are not working independently. Generative capability implicitly helps to capture the relationship between $\theta_{t}$ and $\theta_{S}$ better with more diverse feature. So when two approaches are combined, we could get an approximation of posterior  $p(\theta|D_{1:t})$ with more information on the object without introducing more model or data complexity.

\subsection{Comparison to Generative Model}
It is also important to compare BGR and VGR, which is a representative algorithm using generative models, in details since both are generative models. Results are shown in Figure \ref{fig:vgr}. The performance on CUB and Split-MNIST are similar. The biggest difference is that for VGR to work, it needs to use GAN with much larger parameters. BGR uses only 220k parameters on Split-MNIST and if we choose GAN in VGR with similar size, the perform drops to 88 (VGR-SMALL). Also notice that in VGR paper \cite{farquhar2019unifying}, authors continue to use classifier from previous task. If we re-initiate the classifier every time a new task comes in and train the model with replayed data, then the performance drops to 66 (VGR-GAN-ONLY). This again shows that generative capability works better when combined with discriminative classifier as in BGR. On the other hand, the benefit of using separate generative model as in VGR is that the training speed is fast, which BGR need to generate samples by Langevin dynamics. BGR takes 3097 seconds to finish and VGR only takes 564 seconds. 

\begin{figure}[t!]
\centering
    \includegraphics[width=1.06\linewidth]{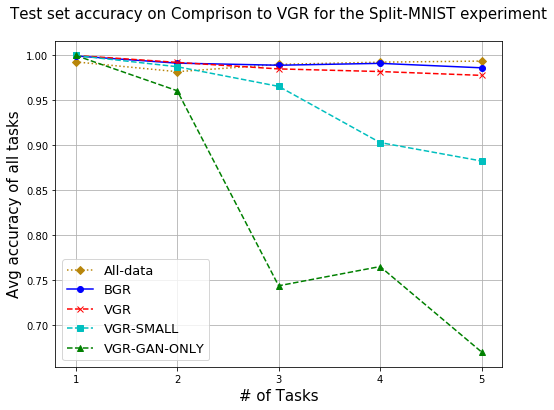}
    \caption{Detailed Comparison of VGR on Split-MNIST. }
    \label{fig:vgr}
\end{figure}

\subsection{Comparison to Memory-based Methods}
Our work focuses on memory-less setup to study how generative capability can empower the model to overcome catastrophic forgetting. However, many large-scale catastrophic forgetting problems are solved by memory-based solutions. Memory-based method allows the algorithm store some data from previous tasks and re-use the data to fine-tune the model in the later stages. It is thus hard to directly compare BGR with memory-based methods since the underlying assumptions are different. However, we provide some results on MNIST and Fashion-MNIST for the completeness, and this also gives reader a change to observe the trade-off between memory and memory-less methods. We run the code of GDumb method \cite{prabhu2020gdumb}\footnote{\url{https://github.com/drimpossible/GDumb}}, and found out that indeed it achieves better MNIST results (98.5) over us (98.2) with memory $k = 5000$. However, it does not perform equivalently well on FashionMNIST. It requires memory $k=7500$ to achieve similar performance (97.2) whereas BGR requires no extra storage.

\subsection{Extension to Class Incremental Learning Setup}

With the advance of catastrophic forgetting research, many other learning scenarios are introduced. Class incremental Learning \cite{prabhu2020gdumb,farquhar2019unifying}, a popular setup, consider a combination of single-head and multi-head setup. It is a single-head setup but each time we will only be given a subset of classes as in multi-head setup. In \cite{farquhar2019unifying}, authors pointed out that this task is challenging for methods without memory systems. Indeed, directly applying the proposed method cannot yield a good result. However, in this section we show that the proposed BGR can be combined with memory-based methods to achieve better results. Specifically, we consider a combination of BGR with an meta-learning method iTAML \cite{rajasegaran2020itaml}. Under the original learning scheme in iTAML, we additionally add the generative regularization term into the training objective, and observe that on split-MNIST dataset, the performance improves from 97.8 (iTAML) to 98.4 (iTAML+Ours), which shows that BGR can be used as a regularization on top of state-of-the-art memory-based methods. Further study of how the proposed BGR could be combined with other methods is an interesting future direction.

\subsection{Qualitative Analysis: Generative Capability Learns Diverse Features}

\begin{figure}[!t]
    \centering
    \begin{subfigure}[h]{0.75\linewidth}
        \includegraphics[width=\linewidth]{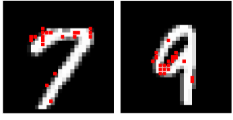}
        \caption{Top 20$\%$ salient points of models trained with SGD. Salient points concentrate on most discriminative part of the digits and model suffers from catastrophic forgetting. Accuracy drops from 99.7$\%$ to 54.2$\%$ after training on another task.}
        \label{fig:SGDpoint}
    \end{subfigure}~\\
  
    \begin{subfigure}[h]{.75\linewidth}
        \includegraphics[width=\linewidth]{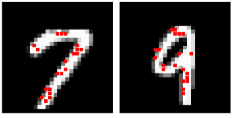}
        \caption{Top 20$\%$ salient points of models trained with the proposed method. EBM provides a generative capability so the salient points scatter equally over the whole stroke of digits. Accuracy drops only from 99.7$\%$ to 95.0$\%$ after training on another task. This shows the importance of generative term in overcoming catastrophic forgetting.  }
        \label{fig:OurPoint}
    \end{subfigure}
    \caption{Illustration of importance of learning diverse features by proposed generative term in the model. }
    \label{fig:intuition}
\end{figure}

Since we add the generative regularization in the objective function, we are interested in whether the learned features are indeed sufficient with generative ability.
In order to generate an object, generative model should not only recognize certain parts of the object but also capture most variations of it. Thus, we hypothesize that generative capability learn a more diverse feature instead of concentrated discriminative features. Such holistic feature capture should prevent model from focusing on only part of object and alleviate the catastrophic forgetting. To get a qualitative assessment, we performed the integrated gradients~\cite{sundararajan2017axiomatic} to investigate which pixels of the image contribute most to the output of the model\footnote{\url{https://github.com/chihkuanyeh/saliency_evaluation}}. The salient points are marked with red dots in~Figure~\ref{fig:intuition}. These points represent the part of the object wit strongest response to the feature extraction process (i.e., activations in the neural network).

As shown in Figure~\ref{fig:SGDpoint}, instead of understanding full stroke of the drawing, training the pair discriminatively with NLL loss makes the model focusing on certain part of the underlying object. Specifically, salient points of the digit 7 are spread mostly on top horizontal stroke, and salient points of the digit 9 centered on lower left curved stroke. Admittedly, these salient points mark the most critical difference between the shape of these two digits so discriminative models can exploit such informative feature to succeed in classification task. However, since not many features are extracted in the first task, when the model moves to the next task, the discriminative model might focus on a very different set of features such that minor adjustment of model parameters will cause the desired feature extraction in task 1 lost. On the other hand, as shown in the in Figure ~\ref{fig:OurPoint}, salient points of BGR with generative capability will be be equally distributed to different parts of the stroke. And the accuracy drops only from 99.7$\%$ to 95.0$\%$ after training on another task. This shows the importance of generative term in overcoming catastrophic forgetting. 

\section{Conclusions}
In this paper, we use carefully analyze the drawback of VCL caused by the mean-field approximation and introduce sufficient generative condition to revalid the factorization to overcome the catastrophic forgetting. To implement the generative ability of partial neural network, we solicit the energy-based models view of neural network. 
Extensive experimental results show that when generative capability  combines with Bayesian inference framework, it can alleviate catastrophic forgetting \emph{significantly} without modifying underlying model architecture. The proposed BGR outperforms state-of-the-art method on Fashion-MNIST dataset about $15\%$ accuracy and CUB dataset about $10\%$.

\subsection*{Acknowledgement}

Part of this work was done during PHC’s internship at Google. CJH and PHC are partially supported by NSF under IIS-1901527, IIS-2008173 and IIS-2048280. 

\bibliography{example_paper}
\bibliographystyle{icml2021}

\clearpage
\appendix

\onecolumn

\section{Proof of Theorem~\ref{thm:gradoflogp}}\label{appendix:A}

\def\thetheorem{\ref{thm:gradoflogp}}
\begin{theorem}

Given a discriminative model $f_{\theta}(x)$, the unbiased gradient estimator of corresponding Energy-based model $\log p_{\theta}(x)$  is given by

\begin{align*}
      \EE_{p_{\theta}(y|x)}[y^\top \nabla_\theta f_{\theta}(x)] -  \EE_{p_{\theta}(x,y)}[y^\top \nabla_\theta f_{\theta}(x)].
\end{align*}
\end{theorem}
\addtocounter{theorem}{-5}
\begin{proof}
Notice that we could derive ELBO of $\log p_{\theta}(x)$ as:
    \begin{align}
    \log p_{\theta}(x) \geq \EE_{q(y|x)}\big[\log\frac{p_{\theta}(x,y)}{q(y|x)}\big] ,
\end{align}
and we know the maximal would be obtained when $KL(q(y|x)\|p_{\theta}(y|x)) = 0$, which implies that optimal $q^{*}(y|x)$ is $p_{\theta}(y|x)$. Thus, we will have 

\resizebox{1\linewidth}{!}{
  \begin{minipage}{\linewidth}
\begin{align}
    \log p_{\theta}(x) &= \EE_{p_{\theta}(y|x)}\big[\log \frac{p_{\theta}(x,y)}{p_{\theta}(y|x)}\big] \\
    &= \EE_{p_{\theta}(y|x)}\big[\log p_{\theta}(x,y) - \log p_{\theta}(y|x)\big] \\
    &= \EE_{p_{\theta}(y|x)}\big[\log \frac{\exp(y^\top f_{\theta}(x))}{\int_{x}\sum_{y}\exp(y^\top f_{\theta}(x))} - \log \frac{\exp(y^\top f_{\theta}(x))}{\sum_{y} \exp(y^\top f_{\theta}(x))}\big] \\
    &= \EE_{p_{\theta}(y|x)}\big[
    \log \sum_{y} \exp(y^\top f_{\theta}(x)) - \log \int_{x}\sum_{y}\exp(y^\top f_{\theta}(x))\big]. \label{eq:logx}
\end{align}
  \end{minipage}
}

Thus, we could obtain $\nabla_{\theta} \log p_{\theta}(x)$ by taking derivative of eq~\eqref{eq:logx}:
 
 \resizebox{1\linewidth}{!}{
  \begin{minipage}{\linewidth}
\begin{align}
    &\nabla_{\theta} \log p_{\theta}(x) \\
    &= \EE_{p_{\theta}(y|x)}\big[ \nabla_{\theta} \big[
    \log \sum_{y} \exp(y^\top f_{\theta}(x)) \big] - \nabla_{\theta} \big[\log \int_{x}\sum_{y}\exp(y^\top f_{\theta}(x)) ] \big]  \\
 &=  \EE_{p_{\theta}(y|x)}\big[
 \frac{\nabla_{\theta}(\sum_{y} \exp(y^\top f_{\theta}(x))}{  \sum_{y} \exp(y^\top f_{\theta}(x))} -  \frac{\nabla_{\theta}(\int_{x}\sum_{y}\exp(y^\top f_{\theta}(x))}{\int_{x}\sum_{y}\exp(y^\top f_{\theta}(x)) }\big] \\
 &=  \EE_{p_{\theta}(y|x)}\big[
 \frac{\sum_{y} \nabla_{\theta}(\exp(y^\top f_{\theta}(x))}{  \sum_{y} \exp(y^\top f_{\theta}(x))} -  \frac{\int_{x}\sum_{y}\nabla_{\theta}(\exp(y^\top f_{\theta}(x))}{\int_{x}\sum_{y}\exp(y^\top f_{\theta}(x)) }\big] \\
  &=  \EE_{p_{\theta}(y|x)}\big[
 \frac{\sum_{y}\exp(y^\top f_{\theta}(x)
\nabla_{\theta}(\log \exp(y^\top f_{\theta}(x))}{  \sum_{y} \exp(y^\top f_{\theta}(x))
 } -  \frac{\int_{x}\sum_{y}\nabla_{\theta}(\exp(y^\top f_{\theta}(x))}{\int_{x}\sum_{y}\exp(y^\top f_{\theta}(x)) }\big] \\
   &=  \EE_{p_{\theta}(y|x)}\big[
   \sum_{y}p_{\theta}(y|x)y^\top  \nabla_\theta f_{\theta}(x)
 -  \frac{\int_{x}\sum_{y}\exp(y^\top f_{\theta}(x)\nabla_{\theta}(\log \exp(y^\top f_{\theta}(x))}{\int_{x}\sum_{y}\exp(y^\top f_{\theta}(x)) }\big] \\
    &=  \EE_{p_{\theta}(y|x)}\big[
   \sum_{y}p_{\theta}(y|x)y^\top  \nabla_\theta f_{\theta}(x)
 -  \int_{x}\sum_{y}p_{\theta}(x,y)y^\top  \nabla_\theta f_{\theta}(x) \big] \\
 &=\EE_{p_{\theta}(y|x)}\big[
 \EE_{p_{\theta}(y|x)}[y^\top \nabla_\theta f_{\theta}(x)] -  \EE_{p_{\theta}(x,y)}[y^\top \nabla_\theta f_{\theta}(x)]\big] \\
  &=\EE_{p_{\theta}(y|x)}[y^\top \nabla_\theta f_{\theta}(x)] -  \EE_{p_{\theta}(x,y)}[y^\top \nabla_\theta f_{\theta}(x)].
 \end{align}

  \end{minipage}
}

Notice that the outer expectation could be taken off since after inner expectation, there won't be any randomness on $y$.
\end{proof}

\section{Proof of Theorem~\ref{thm:gradient_est}}\label{appendix:B}
\def\thetheorem{\ref{thm:gradient_est}}
\textbf{Theorem 3}
The estimation of gradient of loss used in training the proposed method eq \eqref{eq:totalObj} is given by
\begin{align*}
     & \nabla_\theta L(\theta; p_\theta) \nonumber \triangleq  
     \nabla_{\theta}KL(q_{t}(\theta | D_{1:t})  \| q_{t-1}(\theta | D_{1:t-1})) \\ 
&  \nonumber -\nabla_{\theta} \log p_\theta(y | x)   +\gamma(\EE_{p_{\theta}(x,y)} \big[ y^\top \nabla_\theta f_{\theta}(x) \big] - \EE_{D_t}[{y}^\top  \nabla_\theta f_{\theta}(x)]). 
\end{align*}

\begin{proof} 
To solve the objetive function:
\begin{equation}
        \min_{q_{t} \in \bQ} \EE_{q_t,D_t} \big[ -(1 - \lambda) \log p_\theta(y, x) -\lambda \log p_\theta(y | x) 
        - \lambda\log p_{\theta}(x)+ KL\big[q_{t}(\theta | D_{1:t})  \| q_{t-1}(\theta | D_{1:t-1}) \big], \nonumber
\end{equation}

we need to take derivative over the above equation and calculate the unbiased gradient estimator of each term. Notice that the first term is provided in \eqref{eq:CD} and the second term is provided in Theorem \ref{thm:gradoflogp}. Thus, we could substitute these values into the equation to get:
\begin{flalign}
 \nabla_\theta L(\theta; p_\theta) \nonumber 
&\triangleq  - (1-\lambda) \nabla_{\theta}\log p_\theta(y, x) -  \lambda \nabla_{\theta} \log p_\theta(y | x) \nonumber 
     - \lambda \nabla_{\theta}\log p_{\theta}(x)+ \\ & \nabla_{\theta}KL(q_{t}(\theta | D_{1:t})  \| q_{t-1}(\theta | D_{1:t-1})) \nonumber\\ 
        =& \nonumber (1-\lambda)(\EE_{p_{\theta}(x,y)} \big[ y^\top \nabla_\theta f_{\theta}(x) \big] -  \EE_D \big[ y^\top \nabla_\theta f_{\theta}(x) \big] )\\\nonumber
      &- \lambda \nabla_{\theta} \log p_\theta(y | x)  + \nabla_{\theta}KL(q_{t}(\theta | D_{1:t})  \| q_{t-1}(\theta | D_{1:t-1}))\\\nonumber
      &+ \lambda \EE_{p_{\theta}(x,y)}[y^\top \nabla_\theta f_{\theta}(x)])
      - \lambda
      (\EE_{p_{\theta}(y|x)}[y^\top \nabla_\theta f_{\theta}(x)] \\
       =& \nonumber  \nabla_{\theta}KL(q_{t}(\theta | D_{1:t})  \| q_{t-1}(\theta | D_{1:t-1}))
  \nonumber -\lambda \nabla_{\theta} \log p_\theta(y | x)   +\EE_{p_{\theta}(x,y)} \big[ y^\top \nabla_\theta f_{\theta}(x) \big]  \\\nonumber
     &- \EE_{x_b\sim D_t}\EE_{y_b\sim \lambda p_\theta(y|x_b) + (1 - \lambda)D_t}[{y_b}^\top  \nabla_\theta f_{\theta}(x_b)] , 
\end{flalign}

where $x_b$ is the training instance sampled from true data distribution $D_{t}$ with $y_b$ sampled from a mixture of conditional $p_\theta(y|x_b)$ and training sets. Again, to generate the samples $(x_t, y_t)$ from the current model, we exploit the hybrid Monte-Carlo \cite{neal2011mcmc}, specifically the Langevin dynamics sampler, as listed in Algorithm~\ref{alg:langevin}.

In this work, we treat the generative term as a regularization to alleviate catastrophic forgetting in discriminative task. Therefore, we relax the constant $\lambda$ and introduce a new hyperparameter $\gamma$ to represent the importance of the generative regularization. For simplicity, we only draw samples from true data distribution instead of a mixture of conditional $p_\theta(y|x_b)$ and training sets. This leads to the final estimation of gradient of loss used in training the proposed method:
\def\thetheorem{\ref{thm:gradoflogp}}

\resizebox{1.05\linewidth}{!}{
  \begin{minipage}{1\linewidth}
\begin{align*}
     & \nabla_\theta L(\theta; p_\theta) \nonumber \triangleq  
     \nabla_{\theta}KL(q_{t}(\theta | D_{1:t})  \| q_{t-1}(\theta | D_{1:t-1})) \\ 
&  \nonumber -\nabla_{\theta} \log p_\theta(y | x)   +\gamma(\EE_{p_{\theta}(x,y)} \big[ y^\top \nabla_\theta f_{\theta}(x) \big] - \EE_{D_t}[{y}^\top  \nabla_\theta f_{\theta}(x)]). \nonumber\\
\end{align*}
  \end{minipage}}
\end{proof}

\section{Examples of Generated Images}

We show samples of the generated MNIST digit in Figure~\ref{fig:genmnist} and samples of the generated Fashion-MNIST in Fiagure~\ref{fig:genfashion}. These images are generated by using Multi-layer Perceptron model (MLP) with 2 hidden layers and each layer has dimension 256 with ReLU activation function. Notice that our main task is to overcome catastrophic forgetting but not image generation. Generative capability is just used as a regularization term so the images generated are not perfectly following the true data distribution. In addition, we are using a rather small model to build EBM. In practice, people reported to use much larger networks (about 20 times more parmeters) in order to generate more clear images for CIFAR-10 dataset~\cite{du2019implicit}. 

\begin{figure}[h]
    \centering
    \includegraphics[width=.7\linewidth]{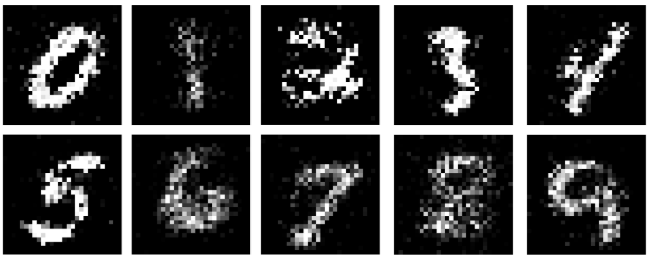}
    \caption{Examples of generated MNIST images. The first row shows digits 0 to 4 and the second row shows digits 5 to 9. }
    \label{fig:genmnist}
\end{figure}

\begin{figure}[h]
    \centering
    \includegraphics[width=.7\linewidth]{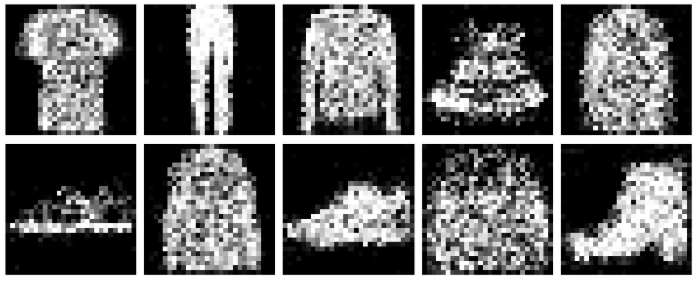}
    \caption{Examples of generated Fashion-MNIST images. The first row shows digits 0 to 4 and the second row shows digits 5 to 9. The first row corresponds to objects T-shirt, Trouser, Pullover, Dress and Coat. The second row corresponds to Sandal, Shirt, Sneaker, Bag and Ankle boot.
}
    \label{fig:genfashion}
\end{figure}
\section{Preprocess of Data}
For all the dataset, we normalized the pixel values in range [0,1]. For MNIST and Fashion-MNIST , we have the train, validation and test splits provided within the dataset. Images in CUB dataset is rather limited. Most classes have samples less than 100. Therefore, we select only the top 100 classes with more images and randomly pick 40 to form the train set and 10 to be validation set. The rest (non-fixed number) of the remaining images will be left as test set. In addition, CUB dataset provided foreground and background segmentation. We segment only the foreground bird images and left the background to be black. Without this, EBM will try to generate background istead and this will not benefit to overcoming forgetting.
\section{Implementation Details}
For each dataset/task, we compare these methods under the same network architecture. 
As illustrated in the related work, we basically do not compare memory-based and model adaption methods as the data and model complexity will increase, but we include generative-based model VGR to compare. The implementation of VGR is primarily based on these two repositories with adjustment toward our setup\footnote{https://github.com/nbro/Continual-learning-1,https://github.com/GMvandeVen/continual-learning}.  For EWC and VCL,  we follow the released open source implementation \footnote{https://github.com/nvcuong/variational-continual-learning}. The chosen baseline methods represent the state-of-the-art algorithms to overcome forgetting without changing model or adding data. For Permuted-MNIST and Split-MNIST, we use a Multi-layer Perceptron model (MLP) with 2 hidden layers and each layer has dimension 256. ReLU is used as the activation function. For Permuted-MNIST,  we use single-head model and for Split-MNIST we use multi-head model. For Fashion-MNIST dataset, we evaluate the results on Convolutaional Neural Networks (CNN) with 4 layers of convolutional layer (32,1), (64,32), (64,64), (64,64) followed by one layer of fully connected layer. For CUB dataset, we apply a Wide-Residual Network  \cite{zagoruyko2016wide} implemented with depth 16 and widen-factor 2. The implementation could be found on the official Pytorch repository\footnote{https://github.com/meliketoy/wide-resnet.pytorch}.  All the models are trained with an ADAM optimizer.

The hyperparameters used in the experiment are listed in the Table~\ref{tab:hyper}. In addition, after each step of SGLD sampling, we will clamp the sample within range [0,1] to make sure the generated image is within the range of true data distribution. Learning rate, Adam Beta, SGLD step size, SGLD noise follows previous implementation of SGLD sampling \footnote{https://github.com/rosinality/igebm-pytorch} or WRN model \footnote{https://github.com/kibok90/iccv2019-inc}.
Number of models sampled from the Bayesian posterior is mostly limited by time constraints. In general we found out 3 is enough but the more the better. For Epochs of each round of SGLD update, more steps is better, but more updates will also lead to very long training time. Thus, in practice we try some numbers from 10 to 100 steps on small portion of data. We will stop searching bigger numbers once the model could generate images look similar to data distribution.
Buffer reinitialization rate is determined from validation set. We search over .05, .2 and .5.  

\begin{table}[ht!]
    \centering
    \begin{tabular}{|c|c|c|c|c|} \hline
          &  Permuted & Split & Fashion & CUB  \\\hline
         Learning Rate & 1e-3 & 1e-3&1e-3&1e-4 \\\hline
         Adam Beta &  (0,0.999) & (0,0.999) & (0,0.999) & (.9, .999)  \\\hline
         Number of models sampled from $p(\theta)$ &  10 & 10 & 10 &3 \\\hline

         Generation Importance $\gamma$ &1 & 1 & 1 & .2 \\\hline

         Buffer Size &10000 & 10000 & 10000 & 200  \\\hline
         SGLD step size & 10 & 10 & 10 & 1 \\ \hline
         Buffer Reinitialization Rate & .05 & .5 & .05 & .05 \\ \hline
         SGLD noise & 5e-3 & 5e-3 & 5e-3 & 1e-2 \\ \hline
         Epochs of each round of SGLD update &  60 & 60 & 5 &20 \\\hline
    \end{tabular}
    \caption{Summarization of hyperparameters used in each task.}
    \label{tab:hyper}
\end{table}

\end{document}